\documentclass[11pt, notitlepage]{article}   % The document class has to be "report"

\usepackage[title]{appendix}

\usepackage{amsmath,amsthm,amsfonts,hyperref}   % These are to write mathematics

\usepackage{bbm} % for indicator function
\usepackage{amscd, array}
\usepackage{amsfonts}
\usepackage{amssymb}
\usepackage{color}      % Need the color package
\usepackage{epsfig}
\usepackage{graphicx}           % Graphics package
\usepackage{algorithm}
\usepackage[noend]{algpseudocode}

\usepackage{tikz}
\usetikzlibrary{shapes}
\usetikzlibrary{matrix}
\usetikzlibrary{positioning}
\usetikzlibrary{fit}

\tikzstyle{res}=[circle,thick,minimum size=4mm,draw=black,fill=red,inner sep=1pt]
\tikzstyle{non-res}=[circle,thick,minimum size=4mm,draw=black,inner sep=1pt]
\tikzstyle{light-res}=[circle,thick,minimum size=4mm,draw=black,fill=red!40,inner sep=1pt]
\tikzstyle{blue}=[circle,thick,minimum size=4mm,draw=black,fill=blue!20,inner sep=1pt]

\newtheorem{theorem}{Theorem}[section]
\newcommand{\optcs}{\textnormal{opt}_{\textnormal{cs}}}
\newcommand{\optas}{\textnormal{opt}_{\textnormal{as}}}
\newcommand{\optag}{\textnormal{opt}_{\textnormal{ag}}}
\newcommand{\optbs}{\textnormal{opt}_{\textnormal{bs}}}
\newcommand{\optambr}{\textnormal{opt}_{\textnormal{amb,r}}}

\newcommand{\optstd}{\textnormal{opt}_{\textnormal{std}}}
\theoremstyle{remark}

\theoremstyle{theorem}

\newtheorem{cor}[theorem]{Corollary}

\newtheorem{lemma}[theorem]{Lemma}
\newtheorem{thm}[theorem]{Theorem}

\newcommand{\cartr}{\textnormal{CART}_{\textnormal{r}}}
\newcommand{\compose}{\textnormal{COMPOSE}}

\def\finf{\mathop{{\rm I}\kern -.27 em {\rm F}}\nolimits}

\usepackage[colorinlistoftodos,textsize=tiny]{todonotes}
\newcommand{\Comments}{1}
\newcommand{\mynote}[2]{\ifnum\Comments=1\textcolor{#1}{#2}\fi}
\newcommand{\mytodo}[2]{\ifnum\Comments=1%
  \todo[linecolor=#1!80!black,backgroundcolor=#1,bordercolor=#1!80!black]{#2}\fi}

\begin{document}

\title{Bounds on the price of feedback for mistake-bounded online learning}
\author{Jesse Geneson and Linus Tang}

\maketitle

\begin{abstract}
We improve several worst-case bounds for various online learning scenarios from (Auer and Long, Machine Learning, 1999). In particular, we sharpen an upper bound for delayed ambiguous reinforcement learning by a factor of 2 and an upper bound for learning compositions of families of functions by a factor of 2.41. We also improve a lower bound from the same paper for learning compositions of $k$ families of functions by a factor of $\Theta(\ln{k})$, matching the upper bound up to a constant factor. In addition, we solve a problem from (Long, Theoretical Computer Science, 2020) on the price of bandit feedback with respect to standard feedback for multiclass learning, and we improve an upper bound from (Feng et al., Theoretical Computer Science, 2023) on the price of $r$-input delayed ambiguous reinforcement learning by a factor of $r$, matching a lower bound from the same paper up to the leading term.
%\blindtext
\end{abstract}

\section{Introduction}

%\blindmathpaper

In this paper, we investigate the mistake-bound model for online reinforcement learning, which was introduced in \cite{angluin, littlestone}. In particular, we focus on several variants of the mistake-bound model which have previously been investigated in \cite{auer, long, geneson, feng}. In all models there is a learner, an adversary, and a family $F$ of functions with some domain $X$ and codomain $Y$. The adversary picks a function $f \in F$ without telling the learner, and starts giving the learner inputs $x_1, x_2, \dots \in X$. The learner guesses the value of $f(x_i)$ for each $i$ and gets some feedback from the adversary. The learner’s goal is to make as few mistakes as possible, while the adversary’s goal is to force the learner to make as many mistakes as possible. 

In the \text{standard} model, the adversary tells the learner the correct value of $f(x_i)$ for each $i$ right after the learner guesses the value. We let $\optstd(F)$ denote the maximum possible number of mistakes that the learner makes in the standard model if both learner and adversary play optimally. In the \text{bandit} model (see, e.g., \cite{acfs, almw, bubeck, chu, crammer, dani, hazan, kakade}), the adversary only tells the learner YES or NO for each $i$ right after the learner guesses the value of $f(x_i)$. We let $\optbs(F)$ denote the maximum possible number of mistakes that the learner makes in the bandit model if both learner and adversary play optimally.

Long \cite{long} proved that if $F$ is a family of functions with codomain of size $k$, then $\optbs(F) \le (1+o(1)) (k \ln{k}) \optstd(F)$, improving a bound of Auer and Long \cite{auer}. Moreover, in the same paper Long showed that there exist families of functions $F$ with codomain of size $k$ such that $\optbs(F) \ge (1-o(1))(k \ln{k}) (\optstd(F)-2)$, where the $o(1)$ is with respect to $k$. As $\optstd(F)$ and $k$ both increase, the ratio of the upper and lower bounds converges to $1$. However, for small values of $\optstd(F)$, there is a substantial gap between the bounds. In particular, when $\optstd(F) = 1$, the maximum possible value of $\optbs(F)$ is $k-1$ \cite{long}. When $\optstd(F) = 3$, there is a multiplicative gap of approximately $3$ between the upper and lower bounds. Long \cite{long} posed the problem of determining the maximum possible value of $\optbs(F)$ when $\optstd(F) = 2$. Before our paper, the best known upper bound for this problem was Long’s upper bound of $2 (1+o(1))k \ln{k}$ \cite{long} and the best known lower bound was $k-1$. In this paper, we improve the lower bound to $\Omega(k \ln{k})$ and we improve the upper bound to $(1+o(1))k \ln{k}$.

We also consider a model where the learner receives delayed ambiguous reinforcement from the adversary. It is called the \text{$r$-delayed ambiguous reinforcement model}. More specifically, in each round $i$ of the learning process, the adversary gives the learner $r$ inputs $x_{i,1}, x_{i,2}, \dots, x_{i,r}$. The learner guesses the correct value of $f(x_{i,j})$ for each $j = 1, 2, \dots, r$, and the learner does not receive the next input $x_{i,j+1}$ until they have guessed the value of $f(x_{i,j})$ for each $j = 1,2,\dots,r-1$. After the learner has finished providing their $r$ guesses, the adversary tells them YES if all guesses were correct and NO otherwise. We let $\optambr(F)$ denote the maximum possible number of mistakes that the learner makes in the $r$-delayed ambiguous reinforcement model if both learner and adversary play optimally.

There is a learning model which sounds very similar to $r$-delayed ambiguous reinforcement, but it can produce far fewer mistakes when both learner and adversary play optimally. In particular, given a family of functions $F$ with domain $X$ and codomain $Y$, let $\cartr(F)$ denote the family of functions $g: X^r \rightarrow Y^r$ for which there exists some $f \in F$ such that $g(x_1, x_2, \dots, x_r) = (f(x_1),f(x_2),\dots,f(x_r))$ for every $(x_1,x_2,\dots,x_r) \in X^r$. Note that learning $\cartr(F)$ in the bandit model is like learning $F$ in the $r$-delayed ambiguous reinforcement model, except the learner receives all of the inputs $x_{i,1}, x_{i,2}, \dots, x_{i,r}$ at the start of round $i$, instead of waiting to receive $x_{i,j+1}$ until they have guessed the value of $f(x_{i,j})$ for each $j = 1,2,\dots,r-1$. Since the learner has less information in the $r$-delayed ambiguous reinforcement learning scenario, it is clear that $\optbs(\cartr(F)) \le \optambr(F)$ for all families of functions $F$.

Feng et al \cite{feng} proved that the maximum possible value of $\frac{\optambr(F)}{\optbs(\cartr(F))}$ over all families of functions $F$ with $|F| > 1$ is $2^{r(1 \pm o(1))}$. They also showed that \[\optbs(\cartr(F)) \le \optambr(F) \leq (1+o(1))\left(k^r r \ln{k}\right)\optstd(F)\] for any set $F$ of functions with domain $X$, codomain $Y$, and $|Y| = k$, where the $o(1)$ is with respect to $k$. Their proof of the upper bound had an error, which we explain and resolve in this paper. Moreover, for every $M > 2r$ they found families of functions $F$ with $\optstd(F) = M$ and \[\optbs(\cartr(F)) \geq (1-o(1))\left(k^r \ln{k}\right)(\optstd(F)-2r).\] We improve the upper bound by a factor of $r$ to obtain \[\optbs(\cartr(F)) \le \optambr(F) \le (1+o(1))\left(k^r \ln{k}\right)\optstd(F).\] Thus, now we can say \[\sup_F \lim_{k \rightarrow \infty} \frac{\optambr(F)}{\left(k^r \ln{k}\right)\optstd(F)} = \sup_F \lim_{k \rightarrow \infty} \frac{\optbs(\cartr(F))}{\left(k^r \ln{k}\right)\optstd(F)} = 1,\] where the maximums are over all families $F$ with codomain of size $k$. 

We also show that the maximum possible value of both $\optambr(F)$ and $\optbs(\cartr(F))$ is $\Theta(k^{\optstd(F)})$ over all families $F$ of functions with codomain of size $k$ and $\optstd(F) \le r$, where the constants in $\Theta(k^{\optstd(F)})$ depend on $r$ and $\optstd(F)$. Note that this generalizes the $r = 1$ result observed by Long \cite{long}. 

In the $k = 2$ case, Auer and Long \cite{auer} proved that $\optambr(F) \leq 2(\ln{2r}) 2^r \optstd(F)$. We show in the $k = 2$ case that $\optambr(F) \leq (\ln{r}) (1+o(1)) 2^r \optstd(F)$ where the $o(1)$ is with respect to $r$. This improves the upper bound of Auer and Long by approximately a factor of $2$.

We also investigate bounds on the difficulty of learning compositions of families of functions in terms of the difficulty of learning the original families. More specifically, suppose that $F_i$  is a family of functions with domain $X$ and codomain $\left\{0,1\right\}$ for each $1 \le i \le k$, and $G$ is a family of functions with domain $\left\{0,1\right\}^{k}$ and codomain $\left\{0,1\right\}$.  Let $\compose(F_1, F_2, \dots, F_k, G)$ denote the family of functions of the form $g(f_1, f_2, \dots, f_k)$ for some $f_1 \in F_1, f_2 \in F_2, \dots, f_k \in F_k, g \in G$. Also, define $\compose(F_1, F_2, \dots, F_k, g) = \compose(F_1, F_2, \dots, F_k, \left\{g\right\})$. Auer and Long \cite{auer} proved that \[\optstd(\compose(F_1, F_2, \dots, F_k, G)) \le 2.41 (1+o(1))(\log_2{k})(\optstd(G)+\sum_{i = 1}^k \optstd(F_i))\] and \[\optstd(\compose(F_1, F_2, \dots, F_k, g)) \le 2.41 (1+o(1))(\log_2{k})\sum_{i = 1}^k \optstd(F_i)\] where the $o(1)$ is with respect to $k$. Moreover, they showed that their upper bounds were within an $O(\log{k})$ factor of optimal. In particular, for any $k \ge 2$ and positive integers $a_1, a_2, \dots, a_k, a_{k+1}$ with $a_{k+1} \le 2^k$, they found families of functions $F_1, F_2, \dots, F_k, G$ with $\optstd(F_i) = a_i$ for each $1 \le i \le k$, $\optstd(G) = a_{k+1}$, and \[\optstd(\compose(F_1, F_2, \dots, F_k, G)) \ge \frac{1}{2} (\optstd(G)+\sum_{i = 1}^k \optstd(F_i)).\] Furthermore, they found families of functions $F_1, F_2, \dots, F_k$ with $\optstd(F_i) = a_i$ and a function $g$ such that \[\optstd(\compose(F_1, F_2, \dots, F_k, g)) \ge \sum_{i = 1}^k \optstd(F_i).\] 
We improve their upper bounds by a factor of 2.41 to obtain \[\optstd(\compose(F_1, F_2, \dots, F_k, G)) \le (1+o(1))(\log_2{k})(\optstd(G)+\sum_{i = 1}^k \optstd(F_i))\] and \[\optstd(\compose(F_1, F_2, \dots, F_k, g)) \le (1+o(1))(\log_2{k})\sum_{i = 1}^k \optstd(F_i).\] Moreover, we show that for any positive integers $k$ and $M$ with $k$ a power of $2$, there exist $k$ families of functions $F_1,F_2,\dots,F_k$ and a function $g$ such that $\optstd(F_i)\leq M$ for all $i\in\{1,2,\dots,k\}$ and
\[\optstd(\compose(F_1, F_2, \dots, F_k, g)) \geq \frac12kM\log_2(k).\] With this new lower bound, we have therefore demonstrated that the upper bound on $\optstd(\compose(F_1, F_2, \dots, F_k, g))$ is within a constant factor of optimal whenever the values of $\optstd(F_i)$ are close to each other for all $1 \le i \le k$. 

Auer and Long also investigated an agnostic learning scenario where the adversary is allowed to give incorrect reinforcement up to $\eta$ times. Define $\optag(F,\eta)$ to be the worst-case number of mistakes that a learner makes in learning $F$ if the adversary can give wrong reinforcement in up to $\eta$ rounds and both learner and adversary play optimally. The adversary only gives weak reinforcement, i.e., a YES or NO. Note that Daniely and Halbertal \cite{daniely} investigated a different model of agnostic learning in which the learning algorithm is evaluated by comparing its number of mistakes to the best possible number of mistakes that can be obtained by applying some function from $F$.

Auer and Long \cite{auer} proved that $\optag(F,\eta) \le 4.82(\optstd(F)+\eta)$ when the codomain of the functions in $F$ has size $2$. Cesa-Bianchi et al \cite{cesa} obtained a modest improvement to this bound by showing that $\optag(F,\eta) \le 4.4035(\optstd(F)+\eta)$ in this case. More recently, Filmus et al \cite{filmus} proved that $\optag(F,\eta) \le 2 \eta + O(\optstd(F)+\sqrt{\optstd(F) \eta})$ in this case, which is sharper than the bound from \cite{cesa} when $\eta$ is sufficently larger than $\optstd(F)$. We prove in general that for all families of functions $F$ with codomain $\left\{0,1, \dots, k-1\right\}$, we have $\optag(F,\eta) \le k \ln{k}(1+o(1))(\optstd(F)+\eta)$, where the $o(1)$ is with respect to $k$. Moreover, we find a family of functions $F$ with codomain $\{0,1,\dots,k-1\}$ such that $\optag(F,\eta)\geq(0.5-o(1))((k\ln k)\optstd(F)+k\eta)$.

In Section~\ref{s:upperstrat}, we describe the general learning strategy that we use for the upper bounds in this paper. In Section~\ref{s:optbsk2}, we prove the new bounds for $\optbs(F)$ when $F$ has codomain $\left\{0,1,\dots,k-1\right\}$ and $\optstd(F) = 2$. In Section~\ref{s:agnostic}, we prove the new bounds for agnostic online learning. In Section~\ref{s:ambr}, we explain an error in the proof of the upper bound on $\optambr(F)$ from Feng et al \cite{feng}. After explaining the error, we improve the upper bound on $\optambr(F)$ by a factor of $r$, using a different method which avoids the error. We also obtain an improvement of Auer and Long's upper bound on $\optambr(F)$ for the $k = 2$ case. In Section~\ref{s:closure}, we prove the new upper and lower bounds for compositions of families of functions. Finally, in Section~\ref{s:discuss}, we discuss some open problems and future directions. 

\section{Upper bound strategy}\label{s:upperstrat}

All of the upper bounds in this paper follow the same kind of strategy, weighted majority voting \cite{auer, LW, long}. In this section, we describe the strategy and introduce some notation which we will use in the proofs of the upper bounds. \\

\noindent {\bf Learning by weighted majority voting.} Suppose that we have some family $F$ of functions, an algorithm $A_s$ which learns $F$ in the standard learning scenario with at most $\optstd(F)$ mistakes, and some learning scenario that is not standard. Using copies of $A_s$, we construct a learning algorithm $A_b$ for the given nonstandard learning scenario. In every learning algorithm $A_b$ that we construct in this paper, we use copies of $A_s$ to vote for each output, where each copy of $A_s$ has a weight and some amount of information about past inputs and outputs. Some or all of the information for a given copy may be incorrect, but we adjust the weights and deactivate copies over time if they make enough mistakes, so that copies with the wrong information will eventually not contribute to the vote.

In particular, we start with a single copy of $A_s$ of weight $1$, which we call $R$. We introduce a parameter $\alpha$, which may depend on $F$ and the learning scenario. Given an input $x$, we feed it to the currently active copies of $A_s$ and they each produce possible outputs. Our learning algorithm $A_b$ returns the output with the greatest weight. If there is a tie for the output with the greatest weight, then we pick the output uniformly at random among the winning outputs. If there are $k$ possible outputs, then the winning output has weight at least $\frac{1}{k}$. 

Whenever the adversary says that $A_b$ is correct, we do nothing: all weights stay the same and all copies still have the same information (they forget the input and output of the current round). When the adversary says that $A_b$ is incorrect, the copies of $A_s$ that voted for the winning output deactivate and produce offspring. Specifically, if some active copy $A$ of $A_s$ had weight $w$ and voted for the winning wrong output, then we split $A$ into a number of new active copies of $A_s$ which have the same memory as $A$, each with weight $\alpha w$, and we make $A$ inactive. Each copy gets a different answer for the output of input $x$ that is consistent with $F$, the information known to $A$, and the results of the current round. The answer will be wrong for all but one of the copies.

If a consistent output does not exist, then $A$ becomes inactive and it produces no offspring. For any active copies of $A_s$ that voted for an output which did not win the vote, we erase their memory of the current round and leave their weights the same. We can view the process as a tree, where $R$ is the root and the nodes at depth $i$ have $i$ mistakes in their memory. When we view the process as a tree, only leaves are active, i.e., only leaves contribute to the vote. 

Throughout the process, there is always some active copy of $A_s$ that has received entirely correct information. Thus, we can conclude that the total weight of the active copies of $A_s$ is at least the weight of this copy, which is at least $\alpha^C$, where $C$ is the maximum possible number of mistakes that any copy of $A_s$ can make if it receives entirely correct information. We bound the total weight of all active copies from above in terms of the number of mistakes made by $A_b$, and then we derive an upper bound on the number of mistakes made by $A_b$ in terms of $\alpha$ and $C$. \\

We use variants of this strategy several times throughout this paper. In particular, we use this exact strategy with different values of $\alpha$ and varying numbers of offspring in our upper bounds for $\optbs(k,2)$ and our upper bounds for agnostic learning. In our upper bounds for delayed ambiguous reinforcement, we use a variant of this strategy, where we change the voting method since each round has $r$ inputs. As for our upper bounds for learning compositions of families of functions, we use another variant of the strategy where we replace the copies of $A_s$ with tuples of copies of learning algorithms for each of the families of functions.

\section{Bounds on $\optbs(k,2)$}\label{s:optbsk2}

In order to state the results in this section more concisely, we introduce some notation. In particular, let $\optbs(k, M)$ denote the maximum possible value of $\optbs(F)$ over all families $F$ of functions with codomain $\left\{0,1,\dots,k-1\right\}$ and $\optstd(F) = M$. Long \cite{long} proved that $\optbs(k, M) \le (1+o(1))k \ln{k} M$ and also showed that $\optbs(k, M) \ge (1-o(1))k \ln{k} (M-2)$, where the $o(1)$ in both bounds is with respect to $k$. When $M = 1$, we have $\optbs(k, 1) = k-1$, while the value of $\optbs(k, 2)$ was an open problem \cite{long}. In particular, it was not known whether $\optbs(k, 2)$ grows on the order of $\Theta(k)$, $\Theta(k \ln{k})$, or something in between.

First, we prove that $\optbs(k,2)=\Theta(k\log k)$. Then, we improve the upper bound for $M = 2$ by showing that $\optbs(k,2)\le k \ln(k) (1+o(1))$. Our strategy in the following lower bound proof shares some similarities with the method that Long \cite{long} employed to prove that $\optbs(k, M) \ge (1-o(1))k \ln{k} (M-2)$, which was also used in \cite{carter, luby, rao1, rao2}. The general strategy is for the adversary to maintain a version space \cite{mitchell} of functions that are consistent with their answers, and to choose inputs in each round that keep the version space as large as possible, regardless of the learner's guesses.

In this section, we also show that $\optbs(k, M)$ is superadditive in $M$ for all fixed $k \ge 1$. This implies by Fekete's lemma that \[\lim_{M \rightarrow \infty} \frac{\optbs(k, M)}{M} = \sup_{M \rightarrow \infty} \frac{\optbs(k, M)}{M} \] for all fixed $k \ge 1$.

In the next proof, we use a family of linear functions to obtain the desired lower bound. In particular, for any prime power $k>1$ and positive integer $n$, define $F_L(k,n)$ to be the set of linear transformations from $\mathbb F_k^n$ to $\mathbb F_k$. Equivalently, for each $\mathbf a\in\mathbb F_k^n$ we let function $f_\mathbf a:\mathbb F_k^n\to\mathbb F_k$ be given by $f_\mathbf a(\mathbf x)=\mathbf a\cdot\mathbf x$. Then $F_L(k,n)=\{f_\mathbf a:\mathbf a\in\mathbb F_k^n\}$. For all prime powers $k>1$ and positive integers $n$, it is easy to see that $\optstd(F_L(k,n))=n$ \cite{almw, blum, shvaytser}. As noted by Long \cite{long}, the learner can make at most $n$ errors by always getting the correct answer whenever the input is in the span of the previous inputs. On the other hand, the adversary can force $n$ errors on any linearly independent set of $n$ inputs by always saying NO to the learner's predicted output for each of these inputs, and then saying that the correct output is either 0 or 1, whichever is different from the learner's prediction. 

\begin{thm}
The asymptotic formula $\optbs(k,2)=\Theta(k\log k)$ holds. 
\end{thm}

\begin{proof}

The upper bound $\optbs(k,2) = O(k\log k)$ follows from Lemma 5 in \cite{long}. We now turn to the lower bound, following a similar approach to \cite{long} and \cite{geneson}. It suffices to provide a strategy for the adversary such that the learner makes at least $\Omega(k\log k)$ mistakes when learning a function in $F_L(k,2)$. To do so, we first introduce a few lemmas.

\begin{lemma}\label{lem:powermean}
Let $K$ be an even positive integer and $a_1,\dots,a_K$ be a sequence of real numbers. Let $s=a_1+\cdots+a_K$ and suppose that $\frac K2$ terms of the sequence have a sum of $\frac{s+t}2$. Then,
\[a_1^2+\cdots+a_K^2\geq\frac{s^2+t^2}K.\]
\end{lemma}

\begin{proof}
For ease of reference, we assume without loss of generality that $a_1+\cdots+a_{K/2}=\frac{s+t}2$. Then, $a_{K/2+1}+\cdots+a_K=\frac{s-t}2$. By the power mean inequality,
\[a_1^2+\cdots+a_{K/2}^2\geq\frac K2\left(\frac{a_1+\cdots+a_{K/2}}{K/2}\right)^2=\frac K2\left(\frac{s+t}K\right)^2\]
\[a_{K/2+1}^2+\cdots+a_K^2\geq\frac K2\left(\frac{a_{K/2+1}+\cdots+a_{K}}{K/2}\right)^2=\frac K2\left(\frac{s-t}K\right)^2\]

Adding these up, we get
\[a_1^2+\cdots+a_K^2\geq\frac K2\left(\frac{s+t}K\right)^2+\frac K2\left(\frac{s-t}K\right)^2=\frac{s^2+t^2}K,\]
as desired.
\end{proof}

\begin{lemma}\label{lem:rn}
Let $k$ be a prime power and $n$ be a positive integer and consider $\mathbf x,\mathbf y\in\mathbb F_k^n$. If $\mathbf x\neq\mathbf y$, then there are exactly $k^{n-1}$ values $\mathbf z\in\mathbb F_k^n$ such that $\mathbf x\cdot\mathbf z=\mathbf y\cdot\mathbf z$.
\end{lemma}

\begin{proof}
We want to show that exactly $k^{n-1}$ values $\mathbf z$ satisfy the condition $\mathbf x\cdot\mathbf z=\mathbf y\cdot\mathbf z$, which is equivalent to $f_{\mathbf x-\mathbf y}(z)=0$. Since $\mathbf x-\mathbf y\neq\mathbf0$, the range of $f_{\mathbf x-\mathbf y}(z)$ has dimension $1$. By the Rank-Nullity Theorem, the kernel of $f_{\mathbf x-\mathbf y}(z)$ has dimension $n-1$, so it has $k^{n-1}$ elements, as desired.
\end{proof}

\begin{lemma}\label{lem:24}
Let $k\geq 2$ be a power of $2$ and $S$ be a subset of $\mathbb F_k^2$. Then there exists $\mathbf u\in\mathbb F_k^2$ such that for all subsets $Z\subseteq\mathbb F_k$ with exactly $k/2$ elements,
\[|\{\mathbf s\in S:\mathbf s\cdot\mathbf u\in Z\}|\leq\frac{|S|+k^{1.5}}2.\]
\end{lemma}

\begin{proof}
For each $\mathbf u\in\mathbb F_k^2$ and $z\in\mathbb F_k$, let $q_{\mathbf u,z}$ be the number of elements $\mathbf s\in S$ satisfying $\mathbf s\cdot\mathbf u=z$. Now, suppose for the sake of contradiction that for every $\mathbf u$, there is a $k/2$-element subset $Z_\mathbf u\subseteq\mathbb F_k$ such that
\[|\{\mathbf s\in S:\mathbf s\cdot\mathbf u\in Z_\mathbf u\}|>\frac{|S|+k^{1.5}}2.\]

Preparing to apply Lemma~\ref{lem:powermean}, we notice that for each $\mathbf u$, we have $\sum\limits_{z\in\mathbb F_k}q_{\mathbf u,z}=|S|$ and that some $k/2$ of these terms have a sum of $\frac{|S|+t}2$ with $t>k^{1.5}$. Applying the lemma with $K=k$, $s=|S|$, and $t>k^{1.5}$, we get
\[\sum\limits_{z\in\mathbb F_k}q_{\mathbf u,z}^2>\frac{|S|^2+k^3}{k}.\]

We finish the proof of the lemma by counting in two ways the number, $N$, of ordered triples $(\mathbf x,\mathbf y,\mathbf u)\in S\times S\times\mathbb F_k^2$ satisfying $\mathbf x\cdot\mathbf u=\mathbf y\cdot\mathbf u$.

By Lemma~\ref{lem:rn}, each of the $|S|(|S|-1)$ pairs of distinct $(\mathbf x,\mathbf y)$ can be completed with exactly $k$ values of $\mathbf u$. Each of the $|S|$ pairs $(\mathbf x,\mathbf x)$ can be completed with all $k^2$ possible values of $\mathbf u$. Thus, $N=k|S|(|S|-1)+k^2|S|$.

We now compute $N$ a different way, by doing casework on $\mathbf u$.
\begin{align*}
k|S|(|S|-1)+k^2|S|=N&=\sum_{\mathbf u\in \mathbb F_k^2}\sum\limits_{z\in\mathbb F_k}q_{\mathbf u,z}^2\\
&>\sum_{\mathbf u\in \mathbb F_k^2}\frac{|S|^2+k^3}{k}\\
&=k|S|^2+k^4.
\end{align*}

This implies $k^2|S|-k|S| >k^4$, which gives the desired contradiction, as $|S|\leq k^2$.
\end{proof}

Armed with this lemma, we supply the adversary's strategy. With $k>1$ ranging over powers of 2, the strategy described below forces the learner to make at least $\Theta(k\log k)$ mistakes to learn a function in $F_L(k,2)$. The adversary will maintain a set of functions $F_t\in F_L(k,2)$ that are consistent with the first $t$ answers, always deny the learner's prediction except when this would cause $F_t$ to become empty, and supply inputs according to the following algorithm.

After $t$ answers have been given, where $t$ is a multiple of $\frac k2$, the functions in $F_t$ are still consistent with all previous answers. The adversary constructs $S_t=\{\mathbf s:f_\mathbf s\in F_t\}$, chooses $\mathbf u$ that satisfies the condition of Lemma~\ref{lem:24}, and supplies $\mathbf u$ as the input $\frac k2$ times. Because $\mathbf u$ satisfies the condition of the lemma, the learner can eliminate at most $\frac{|S_t|+k^{1.5}}2$ elements from the space of consistent functions through these $\frac k2$ predictions. Formally, this means that the adversary can guarantee
\[|F_{t+k/2}|\geq|F_t|-\frac{|S_t|+k^{1.5}}2=\frac{|F_t|-k^{1.5}}2.\]

We claim that the adversary can give at least $r=\left\lfloor\frac{\log_2(k)}2\right\rfloor$ such sets of $\frac k2$ inputs, always answering NO without letting $F_t$ become empty. We prove by induction that for all nonnegative integers $i$,
\[|F_{ik/2}|>(2^{r-i}-1)k^{1.5},\]
and the claim follows from using $i=r$.

The base case, $i=0$, is true because the right hand side of the equation is at most $(k^{0.5}-1)k^{1.5}$. The inductive step (from $i$ to $i+1$) can be proven as follows:
\begin{align*}
|F_{(i+1)k/2}|&\geq\frac{|F_{ik/2}|-k^{1.5}}2\\
&>\frac{(2^{r-i}-1)k^{1.5}-k^{1.5}}2\\
&=(2^{r-(i+1)}-1)k^{1.5},
\end{align*}
completing the induction.

Plugging in $i=r$, we find that $F_{rk/2}>0$. Thus, this algorithm forces the learner to make at least
\[\frac{rk}2=\frac{\left\lfloor\frac{\log_2(k)}2\right\rfloor k}2=\Theta(k\log k)\]
mistakes, proving the theorem.

\end{proof}

Note that Lemma 5 of \cite{long} implies that $\optbs(k,2)\le 2 k \ln(k) (1+o(1))$. In the next result, we improve this upper bound by a factor of $2$.

\begin{thm}\label{thm2upper}
For all $k > 0$, we have $\optbs(k,2)\le \lceil k \ln(k) \rceil + 2k$. 
\end{thm}

\begin{proof}
Let $F$ be a family of functions for which $\optstd(F) = 2$. We use the general learning strategy described in Section~\ref{s:upperstrat}, where we split any copy $A$ that voted for the wrong winning output into at most $k-1$ new active copies of $A_s$, and each copy gets a different answer for the output that is consistent with the function family and not equal to the wrong output that won the vote. 

Note that for the family $F$, the tree can have depth at most $2$ since $A_s$ makes at most $2$ mistakes on $F$. The depth-$1$ level has at most $k-1$ copies of $A_s$, each of which are created after $R$ makes a mistake. The depth-$2$ level has at most $(k-1)^2$ copies of $A_s$, since each depth-$1$ copy of $A_s$ has at most $k-1$ depth-$2$ offspring. 

We set $\alpha = \frac{1}{k^2}$. Then there exists $j \le k-1$ such that for each of the $j$ rounds with mistakes $2,\dots,j+1$, at least one of the depth-$1$ nodes splits into depth-$2$ nodes and becomes inactive, and all depth-$1$ nodes split and become inactive by the round with mistake $j+1$. Note that this is because as long as depth-$1$ nodes remain, their votes weigh more than all of the depth-$2$ nodes put together by the choice of $\alpha$.

In the rounds after the round with mistake $j+1$, only depth-$2$ nodes remain, and they all have the same weight. Thus at least a $\frac{1}{k}$ fraction of them are eliminated with each mistake of $A_b$, until at most $k$ depth-$2$ nodes remain. For all mistakes after that, at least one depth-$2$ node is eliminated on each mistake. Thus, $A_b$ makes a total of at most $2k+m$ mistakes, where \[(k-1)^2 \left(1-\frac{1}{k}\right)^m \le k, \text{ i.e. } \left(1-\frac{1}{k}\right)^m \le \frac{k}{(k-1)^2}.\] Since $(1-\frac{1}{k})^m \le e^{-m/k}$, it suffices to take $m = \lceil k \ln(k) \rceil$. Thus, the number of mistakes is at most $2k + \lceil k \ln(k) \rceil$.
\end{proof}

Next, we show that $\optbs(k, M_1+M_2) \ge \optbs(k,M_1)+\optbs(k,M_2)$ for all $M_1, M_2 \ge 1$. Using this, we prove the existence of \[\lim_{M \rightarrow \infty} \frac{\optbs(k, M)}{M}\] for all fixed $k \ge 1$.

\begin{lemma}
    For all fixed $k \ge 1$, $\optbs(k, M)$ is superadditive in $M$.
\end{lemma}

\begin{proof}
Let $F_1$ be a family with domain $X_1$, codomain $\left\{0,1,\dots,k-1\right\}$, $\optstd(F_1) = M_1$, and $\optbs(F_1) = \optbs(k, M_1)$. Let $F_2$ be a family with domain $X_2$ which is disjoint from $X_1$, codomain $\left\{0,1,\dots,k-1\right\}$, $\optstd(F_2) = M_2$, and $\optbs(F_2) = \optbs(k, M_2)$. Define $G$ to be the family of functions $g_{f_1,f_2}$ for $f_1 \in F_1$ and $f_2 \in F_2$ for which $g_{f_1,f_2}(x) = f_1(x)$ if $x \in X_1$ and $g_{f_1,f_2}(x) = f_2(x)$ if $x \in X_2$. Note that \[\optstd(G) = \optstd(F_1)+\optstd(F_2)\] and \[\optbs(G) = \optbs(F_1)+\optbs(F_2).\] 
Indeed, let $A_i$ (resp. $B_i$) be an optimal learning algorithm for $F_i$ in the standard (resp. bandit) scenario. To see that \[\optstd(G) \le \optstd(F_1)+\optstd(F_2)\] and \[\optbs(G) \le \optbs(F_1)+\optbs(F_2),\] note that the learner can use $A_1$ (resp. $B_1$) for inputs from $X_1$ while ignoring information about inputs from $X_2$, and they can use $A_2$ (resp. $B_2$) for inputs from $X_2$ while ignoring information about inputs from $X_1$. To see that \[\optstd(G) \ge \optstd(F_1)+\optstd(F_2)\] and \[\optbs(G) \ge \optbs(F_1)+\optbs(F_2),\] note that the adversary can ask the learner inputs from $X_1$ until they force $\optstd(F_1)$ (resp. $\optbs(F_1)$) mistakes, and then they can ask the learner inputs from $X_2$ until they force another $\optstd(F_2)$ (resp. $\optbs(F_2)$) mistakes. 

Since we have found a family $G$ of functions with codomain $\left\{0,1,\dots,k-1\right\}$, \[\optstd(G) = M_1+M_2,\] and \[\optbs(G) = \optbs(k,M_1)+\optbs(k,M_2),\] we have by definition that \[\optbs(k, M_1+M_2) \ge \optbs(k,M_1)+\optbs(k,M_2).\]
\end{proof}

\begin{cor}\label{fekete}
For all fixed $k \ge 1$, we have \[\lim_{M \rightarrow \infty} \frac{\optbs(k, M)}{M} = \sup_{M \rightarrow \infty} \frac{\optbs(k, M)}{M}. \]    
\end{cor}

\begin{proof}
    This follows from Fekete's lemma since $\optbs(k, M)$ is superadditive in $M$.
\end{proof}

\section{Agnostic learning}\label{s:agnostic}

Recall that Auer and Long \cite{auer} showed that $\optag(F,\eta) \le 4.82(\optstd(F)+\eta)$ for all families of functions $F$ with codomain $\left\{0,1\right\}$, and Cesa-Bianchi et al \cite{cesa} obtained a slight improvement to their bound by lowering the leading constant from 4.82 to 4.4035. In this section, we generalize the bound to families of functions $F$ with codomain $\left\{0,1, \dots, k-1\right\}$. Furthermore, we derive a lower bound on the maximum possible value of $\optag(F, \eta)$ with respect to $\optstd(F)$ and $\eta$ for families of functions $F$ with codomain $\left\{0,1, \dots, k-1\right\}$. We start with a simple proof which yields the same upper bound as \cite{cesa}, and then we modify the proof to obtain our general upper bound.

\begin{thm}
 For all families of functions $F$ with codomain $\left\{0,1\right\}$, we have $\optag(F,\eta) \le 4.4035(\optstd(F)+\eta)$.   
\end{thm}

\begin{proof}
Again, we use the general learning strategy described in Section~\ref{s:upperstrat}. Whenever $A_b$ gets a NO response from the adversary for the output of $f(x)$, each copy of $A_s$ that voted for the output splits into at most two new copies, one that gets $f(x) = 0$ and the other that gets $f(x) = 1$. If the old copy had weight $w$, then each of the new copies has weight $\alpha w$, where $\alpha = 0.1469$. 

First, observe that $2 \alpha < 1$. Thus, if $A_b$ makes a mistake in round $t$, and the total weight of all active copies is $W_t$ right before round $t$, then we have \[W_{t+1} \le \frac{1}{2}W_t+\frac{1}{2} (2\alpha)W_t =  \left(\frac{1}{2} + \alpha\right)W_t.\] Moreover, there must always be some active copy of $A_s$ which receives the correct information, so this copy must have weight at least $\alpha^{\optstd(F)+\eta}$. Thus, for all $t$ we have \[W_t \ge \alpha^{\optstd(F)+\eta}.\] If $m$ denotes the number of mistakes that $A_b$ makes, then we have \[\left(\frac{1}{2} + \alpha\right)^m \ge \alpha^{\optstd(F)+\eta},\] which implies that \[m \le \frac{(\optstd(F)+\eta)\ln{\alpha}}{\ln\left(\frac{1}{2} + \alpha\right)} < 4.4035(\optstd(F)+\eta).\]
\end{proof}

Next we generalize the proof of the last result from $k = 2$ to any $k \ge 2$.

\begin{thm}
 For all families of functions $F$ with codomain $\left\{0,1, \dots, k-1\right\}$, we have $\optag(F,\eta) \le k \ln{k}(1+o(1))(\optstd(F)+\eta)$, where the $o(1)$ is with respect to $k$.   
\end{thm}

\begin{proof}
    We use the same idea as the last proof, except whenever $A_b$ gets a NO response from the adversary for the output of $f(x)$, each copy of $A_s$ that voted for the output splits into at most $k$ new copies, each getting a different possible value of $f(x)$. If the old copy had weight $w$, then each of the new copies has weight $\alpha w$, where $\alpha = \frac{1}{k \ln{k}}$.

    Let $k$ be sufficiently large so that $k \alpha < 1$. If $A_b$ makes a mistake in round $t$, and the total weight of all active copies is $W_t$ right before round $t$, then we have \[W_{t+1} \le \left(\frac{k-1}{k}\right)W_t+\frac{1}{k} (k\alpha)W_t =  \left(\frac{k-1}{k} + \alpha\right)W_t.\] Moreover, there must always be some active copy of $A_s$ which receives the correct information, so this copy must have weight at least $\alpha^{\optstd(F)+\eta}$. Thus, for all $t$ we have \[W_t \ge \alpha^{\optstd(F)+\eta}.\] If $m$ denotes the number of mistakes that $A_b$ makes, then we have \[\left(\frac{k-1}{k} + \alpha\right)^m \ge \alpha^{\optstd(F)+\eta},\] which implies that \[m \le \frac{(\optstd(F)+\eta)\ln{\alpha}}{\ln\left(\frac{k-1}{k} + \alpha\right)} = k \ln{k}(1+o(1))(\optstd(F)+\eta),\] where the $o(1)$ is with respect to $k$.
\end{proof}

\begin{theorem}
For any prime $k \ge 5$, there exists a family of functions $F$ with codomain $\{0,1,\dots,k-1\}$ such that $\optag(F,\eta)\geq(0.5-o(1))((k\ln k)\optstd(F)+k\eta)$.
\end{theorem}

\begin{proof}
Let $k \ge 5$ be prime. By Theorem 1 of \cite{long} and Theorem 2.4 of \cite{geneson}, there exist families of functions $F = F_L(k, M)$ satisfying $\optstd(F) = M$ and $\optag(F,\eta)\geq\optbs(F)\geq (1-o(1))(k\ln k)\optstd(F)$.

Now, we show that $\optag(F_L(k, M),\eta)\geq k\eta$. Indeed, the adversary can ask any query (other than the one corresponding to the zero vector) $k\eta $ times, and always claim that the learner's response is wrong. Since there is a response that was given at most $\eta$ times, the adversary can choose a function for which at most $\eta$ of these responses were lies. Thus, we have $\optag(F_L(k, M),\eta)\geq k\eta$.

Thus, there is a family $F$ that satisfies
\begin{align*}
\optag(F,\eta)&\geq\max((1-o(1))(k\ln k)\optstd(F),k\eta)\\
&\geq(0.5-o(1))((k\ln k)\optstd(F)+k\eta).
\end{align*}
\end{proof}

\section{Delayed ambiguous reinforcement}\label{s:ambr}

Feng et al \cite{feng} proved that for any set $F$ of functions with domain $X$ and codomain $Y$ and for any $r \geq 1$, \[\optambr(F) \leq (1+o(1))\left(|Y|^r r\ln{|Y|}\right)\optstd(F).\] The proof used a weighted majority voting algorithm for the learner, but there was a flaw in the algorithm. Specifically, the learning algorithm makes a prediction for the $r$ outputs in the current round by taking a weighted vote over the predictions of each copy for the outputs of all $r$ inputs. Later, the proof claims that the copies that voted for the winning outputs must account for at least a $\frac{1}{k^r}$ fraction of the total weight of all copies, and then this claim is used to bound the total number of mistakes made by the learning algorithm.

The problem is that since there are multiple inputs in each round, it is possible that the copies that win the vote for one input are disjoint from the copies that win the vote for another input in the same round. Thus, there might not be any copies that voted for all of the winning outputs. Therefore, the claim that the copies that voted for the winning outputs must account for at least a $\frac{1}{k^r}$ fraction of the total weight of all copies is false. One way to fix this error is to only let copies vote for the output of $x_{i+1}$ if they were one of the copies that won the vote for the output of $x_i$. We use this idea in the next proof.

\begin{thm}\label{theorem:ambrlongupper}
For any set $F$ of functions from some set $X$ to $\{0,1,\dots,k-1\}$ and for any fixed $r \geq 1$, we have \[\optambr(F) \leq (1+o(1))\optstd(F)k^r \ln{k},\] where the $o(1)$ is with respect to $k$.
\end{thm}

\begin{proof}
Here we use a variant of the general learning strategy in Section~\ref{s:upperstrat}. Let $k = |Y|$ and $M = \optstd(F)$. Suppose that $k$ is sufficiently large that $k \ln{k} > r(k-1)$. Start with a copy $A_s$ of a learning algorithm for $F$ which is guaranteed to make at most $M$ errors in the standard learning scenario. We refer to this copy as $R$. We construct a learning algorithm $A_b$ for learning $F$ in the $r$-input delayed ambiguous reinforcement model. In each round, suppose that we have $r$ inputs $x_1, \dots, x_r$. For the first input $x_1$, the learning algorithm $A_b$ outputs the winner of a weighted vote of the currently active copies of $A_s$, and for input $x_{i+1}$ with $i > 0$, $A_b$ outputs the winner of a weighted vote of the copies of $A_s$ that won the vote for the output of $x_i$. 

Every time that $A_b$ gets a NO answer from the adversary for the outputs of $x_1,\dots,x_r$ in some round $t$ and an active copy $A$ of $A_s$ with weight $w$ voted for the winning wrong answer $q_1,\dots,q_r$, split $A$ into at most $r(k-1)$ new active copies of $A_s$ which inherit the same memory as $A$ before the round $t$, each with weight $\alpha w$, and make $A$ inactive. At most $k-1$ of these new active copies correspond to each input $x_1,\dots,x_r$. Specifically for each $i$, at most $k-1$ of these copies get a different answer for the output of $x_i$ that is consistent with the function family and not equal to $q_i$. In particular, each copy remembers everything that $A$ remembered before round $t$, in addition to a single input $x_i$ from round $t$, the guess of $A$ for the output of $x_i$ which is treated as a wrong guess in the memory, and a possible output for $x_i$ which may not be the true output but is consistent with the family $F$ and the memory of the copy. For any copy $A'$ spawned from $A$ that remembers the input $x_i$, no other inputs from round $t$ besides $x_i$ are included in the memory of $A'$. 

%If a consistent output does not exist, then $A$ becomes inactive and it produces no offspring. For copies of $A_s$ that did not win the vote for $x_r$ (the final input of the round), rewind them so that they have no recollection of the last round and make no change to their weights. As in Theorem ~\ref{thm2upper}, we can view the process as a tree, where $R$ is the root, the nodes at depth $i$ have $i$ mistakes in their memory, and leaves are the only nodes that contribute to the vote since they are the only active nodes.

Fix $\alpha = \frac{1}{k \ln{k}}$. If $W_t$ is the total weight of the active copies of $A_s$ before round $t$, then we must have $W_t \ge \alpha^{\optstd(F)}$ since at least one active copy of $A_s$ always gets the correct outputs. Indeed, in every round of the learning process, there is some leaf that gets only correct outputs. In particular, when the process starts, the only leaf is $R$ and $R$ gets no information about the outputs. At the start of any round of the process, if $A$ is a copy of $A_s$ which is a leaf in the tree that has only received correct outputs, then $A$ will either be a leaf at the end of the round or it will split into new leaves assuming that it votes for the wrong $r$-tuple of outputs and it has a winning vote. If $A$ votes for the wrong tuple of outputs, then it must get the output wrong for some input $x_i$, which means that one of the new leaves $A'$ that is an offspring of $A$ will get the correct output for $x_i$, as well as all of correct outputs from $A$. Thus, $A'$ only gets correct outputs. 

If $A_b$ gets the wrong output in round $t$, then active copies of $A_s$ with total weight at least $\frac{W_t}{k^r}$ are cloned and deactivated to make at most $r(k-1)$ new active copies which each have weight $\alpha$ times their old weight. Since $\alpha r (k-1) < 1$, this implies that \[W_{t+1} \le \left(1-\frac{1}{k^r}\right)W_t + \frac{1}{k^r}(\alpha r (k-1)W_t) < \left(1-\frac{1}{k^r}\right)W_t + \alpha \frac{r}{k^{r-1}} W_t.\] Thus after $A_b$ has made $m$ mistakes, we have \[W_t \le \left(1-\frac{1}{k^r}+\alpha \frac{r}{k^{r-1}}\right)^m \le e^{-\left(\frac{1}{k^r}-\alpha \frac{r}{k^{r-1}}\right)m}.\] Therefore \[e^{-\left(\frac{1}{k^r}-\alpha \frac{r}{k^{r-1}}\right)m} \ge \alpha^{\optstd(F)},\] so we can bound the number of mistakes by \[m \le \frac{\ln(\frac{1}{\alpha})\optstd(F)}{\frac{1}{k^r}-\alpha \frac{r}{k^{r-1}}} = (1+o(1))\left(k^r \ln{k}\right)\optstd(F).\]
\end{proof}

In the case that $k = 2$, we modify the last proof to obtain an improvement of an upper bound of Auer and Long \cite{auer} by a factor of $2$. 

\begin{thm}
For any set $F$ of functions from some set $X$ to $\{0,1\}$ and for positive integers $r$, we have \[\optambr(F) \leq \optstd(F)2^r \ln{r} (1+o(1)),\] where the $o(1)$ is with respect to $r$.
\end{thm}

\begin{proof}
We use the same strategy as in the last proof, except we let $\alpha = \frac{1}{r \ln{r}}$. Thus, for $r$ sufficiently large we have $\alpha r < 1$, so we obtain \[W_{t+1} \le (1-\frac{1}{2^r})W_t + \frac{1}{2^r}(\alpha r W_t) = \left(1-\frac{1}{2^r}+\frac{\alpha r}{2^r}\right)W_t,\] so we have \[m \le \frac{\ln(\frac{1}{\alpha})\optstd(F)}{\frac{1}{2^r}-\frac{\alpha r}{2^r}} = \optstd(F)2^r \ln{r} (1+o(1)),\] where the $o(1)$ is with respect to $r$.
\end{proof}

In the next result, we consider families of functions $F$ with $\optstd(F) \le r$. For this case, we are able to sharpen the upper bound from Theorem~\ref{theorem:ambrlongupper}.

\begin{thm}
For any fixed $r \ge 1$ and family $F$ with codomain of size $k$ and $\optstd(F) \le r$, we have $\optambr(F) \le (r k)^{\optstd(F)}$. 
\end{thm}

\begin{proof}
Let $M = \optstd(F)$. We use the same learning algorithm as in Theorem~\ref{theorem:ambrlongupper}, except we change the value of $\alpha$ and we use a different analysis of the algorithm. Note that the tree resulting from our learning algorithm can have depth at most $M$ since $A_s$ makes at most $M$ mistakes on $F$. The depth-$i$ row has at most $r^i(k-1)^i$ copies of $A_s$ for each $i \le M$, since each depth-$i$ copy of $A_s$ has at most $r(k-1)$ depth-$(i+1)$ offspring. Moreover, depth-$i$ nodes each have weight $\alpha^i$. 

Let $\alpha = \frac{1}{r^M k^M}$. We define a sequence of integers $j_0, j_1, \dots, j_{M-1}$. Let $j_0$ be the round of the first mistake, i.e., the round where $R$ makes a mistake and splits into depth-$1$ copies of $A_s$. For each $d \le M-1$, after round $j_d$ has occurred, only nodes of depth at least $d+1$ remain. We claim that given $j_d$ with $d \le M-2$, there exists $j_{d+1} > j_d$ such that for each of the rounds with mistakes $j_d+1,\dots,j_{d+1}$, at least one of the depth-$(d+1)$ nodes splits into depth-$(d+2)$ nodes and becomes inactive, and all depth-$(d+1)$ nodes split and become inactive by the round with mistake $j_{d+1}$. Note that this is because as long as depth-$(d+1)$ nodes remain, their votes weigh more than all of the higher depth nodes put together by choice of $\alpha$.

In the rounds after the round with mistake $j_{M-1}$, only depth-$M$ nodes remain, and they all have the same weight. When we view the process as a whole, each mistake splits at least one minimum-depth active node until all active nodes have depth $M$, and then each mistake deactivates at least one node once all active nodes have depth $M$. Thus the total number of mistakes is at most
\[\sum_{j = 0}^{M} r^j (k-1)^j \le (rk)^M.\]
\end{proof}

\begin{thm}
There exist families $F$ with codomain of size $k$ and $\optstd(F) \le r$ for which \[\optbs(\cartr(F))\ge\left(\sum_{i=0}^{\optstd(F)}\binom ri(k-1)^i\right)-1.\] 
\end{thm}

\begin{proof}
Let $F = F_{k,M}$ be the class consisting of all functions $\mathbb Z\to\{0,1,\dots,k-1\}$ that output $0$ on all but at most $M$ inputs. Note that $\optstd(F)\le M$ because of the learner's strategy of guessing $0$ on every new input. Also, $\optstd(F)\ge M$ because of the adversary's strategy of giving $M$ distinct queries and choosing a function such that all $M$ of the learner's responses are wrong.

Now, we prove a lower bound on $\optbs(\cartr(F))$. The adversary can repeatedly give the same batch of $r$ inputs, and claim that the learner's response is wrong until the learner has guessed all possible $r$-tuples of outputs.

The possible $r$-tuples of outputs are the elements of $\{0,1,\dots,k-1\}^r$ containing at most $M$ nonzero elements, and the number of such tuples is
\[\sum_{i=0}^M\binom ri(k-1)^i.\]

Thus, the adversary's strategy of asking about the same $r$-tuple repeatedly proves that the family $F_{k,M}$ satisfies
\[\optbs(\cartr(F))\ge\left(\sum_{i=0}^{\optstd(F)}\binom ri(k-1)^i\right)-1.\]
\end{proof}

\begin{cor}
There exist families $F$ with codomain of size $k$ and $\optstd(F) \le r$ for which $\optbs(\cartr(F)) \ge k^{\optstd(F)}-1$. 
\end{cor}

\begin{proof}
By the above theorem, there exist families that satisfy
\begin{align*}
\optbs(\cartr(F))&\ge\left(\sum_{i=0}^{\optstd(F)}\binom ri(k-1)^i\right)-1\\
&\ge\left(\sum_{i=0}^{\optstd(F)}\binom{\optstd(F)}i(k-1)^i\right)-1\\
&=k^{\optstd(F)}-1.
\end{align*}
\end{proof}

% \begin{thm}
% There exist families of $k$-class classifiers $F$ with $\optstd(F) \le r$ for which $\optambr(F) \ge k^{\optstd(F)}-1$. 
% \end{thm}

% \begin{proof}
% We use the family $F_L(k, M)$ with the same inputs $e_1, e_2, \dots, e_M, e_M, \dots, e_M$ in each round, where $e_M$ occurs a total of $r-M+1$ times in each round. For these inputs, the adversary can force the learner to make $k^M-1$ errors since any possible tuple of the form $(x_1, x_2, \dots, x_M, x_M, \dots, x_M)$ with $x_1, x_2, \dots, x_M \in \left\{0,1, \dots, k-1 \right\}$ is the output of the inputs $(e_1, e_2, \dots, e_M, e_M, \dots, e_M)$ for some function in $F_L(k, M)$.
% \end{proof}

\begin{cor}
For any fixed $r \ge 1$, the maximum possible value of both $\optambr(F)$ and $\optbs(\cartr(F))$ for any family $F$ with codomain of size $k$ and $\optstd(F) \le r$ is $\Theta(k^{\optstd(F)})$, where the constants in the bound depend on $r$ and $\optstd(F)$. 
\end{cor}

\section{Closure bounds}\label{s:closure}

Recall that Auer and Long \cite{auer} showed that \[\optstd(\compose(F_1, F_2, \dots, F_k, G)) \le 2.41 (1+o(1))(\log_2{k})(\optstd(G)+\sum_{i = 1}^k \optstd(F_i)).\] In addition, they found families of functions $F_1, F_2, \dots, F_k$ with $\optstd(F_i) = a_i$ and a function $g$ such that \[\optstd(\compose(F_1, F_2, \dots, F_k, g)) \ge \sum_{i = 1}^k \optstd(F_i).\] In this section, we improve the leading constant in the upper bound from $2.41$ to $1$. In addition, we construct a sequence $F_1,F_2,\dots,F_k$ of classes of functions and a function $g$ such that $\optstd(F_i)\leq M$ for all $i\in\{1,2,\dots,k\}$ and
\[\optstd(\compose(F_1, F_2, \dots, F_k, g)) \geq \frac12kM\log_2(k).\] Note that this improves Auer and Long's lower bound by a factor of $\Theta(\ln{k})$, and shows that the upper bound on $\optstd(\compose(F_1, F_2, \dots, F_k, g))$ is within a constant factor of optimal, whenever the values of $\optstd(F_i)$ are close to each other for all $1 \le i \le k$.

\begin{thm}
For any $k$ families of functions $F_1, F_2, \dots, F_k$ with domain $X$ and codomain $\left\{0,1\right\}$ and for any family of functions $G$ with domain $\left\{0,1\right\}^k$ and codomain $\left\{0,1\right\}$, we have \[\optstd(\compose(F_1, F_2, \dots, F_k, G)) \leq (1+o(1)) \log_2(k)(\optstd(G)+\sum_{i = 1}^k \optstd(F_i)).\]
\end{thm}

\begin{proof}
Here we use a variant of the general learning strategy in Section~\ref{s:upperstrat}. For each $i = 1, 2, \dots, k$, let $A_{s, i}$ be a copy of a learning algorithm for $F_i$ that is guaranteed to make at most $\optstd(F_i)$ errors in the standard learning scenario for $F_i$. Moreover, let $A_{s, G}$ be a copy of a learning algorithm for $G$ that is guaranteed to make at most $\optstd(G)$ errors in the standard learning scenario for $G$. We form a $(k+1)$-tuple $(A_{s, 1}, A_{s,2}, \dots, A_{s, k}, A_{s, G})$ where the $i^{\text{th}}$ coordinate is the learning algorithm $A_{s, i}$ for each $1 \le i \le k$ and the $(k+1)^{\text{st}}$ coordinate is $A_{s, G}$. We refer to the initial tuple as $R$, and none of the algorithms in the tuple $R$ have any information about the outputs of the hidden function $g(f_1, f_2, \dots, f_k)$. We construct a learning algorithm $A_b$ for learning $\compose(F_2, \dots, F_k, G)$ in the standard model. In each round, given an input $x$, the learning algorithm $A_b$ outputs the winner of a weighted vote of the currently active tuples. The $i^{\text{th}}$ coordinate of each tuple votes for a predicted value $q_i$ of $f_i(x)$ for each $1 \le i \le k$, and the $(k+1)^{\text{st}}$ coordinate of each tuple votes for a predicted value $p$ of $g(q_1, q_2, \dots, q_k)$. %In case of ties, the winner is randomly chosen to be either 0 and 1 with equal probability.

Every time that $A_b$ gets a NO answer from the adversary in some round $t$ and an active tuple with weight $w$ voted for the winning wrong answer, split $A$ into at most $k+1$ new active tuples which inherit the same memory as $A$ before the round $t$, each with weight $\alpha w$, and make $A$ inactive. In particular, the coordinates of each new tuple $A'_i$ born from $A$ remember everything that the corresponding coordinates of $A$ remembered before round $t$. In addition, for each $1 \le i \le k$, the $i^{\text{th}}$ coordinate of $A'_i$ remembers the input $x$ from round $t$, the guess $q_i$ of the $i^{\text{th}}$ coordinate of $A$ for the output of $f_i(x)$ which is treated as a wrong guess in the memory, and a possible output for $f_i(x)$ (equal to 0 if $q_i = 1$ and 1 if $q_i = 0$) which might not be the true output but is consistent with the family $F_i$ and $A'_i$'s memory of the wrong guess. If there is no possible output that is consistent with the family $F_i$ and $A'_i$'s memory of the wrong guess, then $A'_i$ does not exist. The $(k+1)^{\text{st}}$ coordinate of the $(k+1)^{\text{st}}$ new tuple $A'_{k+1}$ born from $A$ remembers the input $(q_1, q_2, \dots, q_k)$ to the $(k+1)^{\text{st}}$ coordinate of $A$, the guess $p$ of the $(k+1)^{\text{st}}$ coordinate of $A$ for the output of $g(q_1, q_2, \dots, q_k)$ which is treated as a wrong guess in the memory, and a possible output for $g(q_1, q_2, \dots, q_k)$ (equal to 0 if $p = 1$ and 1 if $p = 0$) which might not be the true output but is consistent with the family $G$ and $A'_{k+1}$'s memory of the wrong guess. If there is no output that is consistent with the family $G$ and $A'_{k+1}$'s memory of the wrong guess, then $A'_{k+1}$ does not exist.

%If a consistent output does not exist for any of $A$'s possible offspring, then $A$ becomes inactive and it produces no offspring. For tuples that did not win the vote, rewind them so that they have no memory of the last round and make no change to their weights. As in the other upper bounds in this paper, we can view the process as a tree, where $R$ is the root, the nodes at depth $i$ have $i$ mistakes in their memory summed over all of the coordinates in their respective tuples, and leaves are the only nodes that contribute to the vote since they are the only active nodes. Note that some leaves might not contribute to the vote, since leaves may become inactive without producing any offspring.

Fix $\alpha = \frac{1}{(k+1) \ln{k}}$. If $W_t$ is the total weight of the active tuples before round $t$, then we must have $W_t \ge \alpha^{\optstd(G)+\sum_{i = 1}^k \optstd(F_i)}$ since at least one active tuple always gets the correct outputs. Indeed, in every round of the learning process, there is some active leaf that gets only correct outputs. In particular, when the process starts, the only active leaf is $R$ and $R$ gets no information about the outputs. At the start of any round of the process, if $A$ is a tuple which is an active leaf in the tree that has only received correct outputs, then $A$ will either be an active leaf at the end of the round or it will split into new active leaves assuming that it votes for the wrong output and it has a winning vote. If $A$ votes for the wrong output and it has a winning vote, then at least one of its coordinates must get the wrong output, which means that one of the new active leaves $A'$ that is an offspring of $A$ will get the correct output, as well as all of the correct outputs from $A$. Thus, $A'$ only gets correct outputs. 

If $A_b$ gets the wrong output in round $t$, then tuples with total weight at least $\frac{W_t}{2}$ are cloned to make at most $k+1$ copies which each have weight $\alpha$ times their old weight. Suppose that $k$ is sufficiently large so that $\alpha (k+1) < 1$. Thus, we have \[W_{t+1} \le \frac{1}{2}W_t + \frac{1}{2}\alpha (k+1) W_t.\] Thus after $A_b$ has made $m$ mistakes, we have \[W_t \le \left(\frac{1}{2}+\frac{1}{2 \ln{k}}\right)^m.\] Therefore \[\left(\frac{1}{2}+\frac{1}{2 \ln{k}}\right)^m \ge \alpha^{\optstd(G)+\sum_{i = 1}^k \optstd(F_i)},\] so we can bound the number of mistakes by \begin{align*}\begin{split} m &\le (1+o(1))\log_2\left(\frac{1}{\alpha}\right)(\optstd(G)+\sum_{i = 1}^k \optstd(F_i)) \\
&= (1+o(1)) \log_2(k)(\optstd(G)+\sum_{i = 1}^k \optstd(F_i)).\end{split}\end{align*}
\end{proof}

\begin{cor}\label{cor:compose}
For any $k$ families of functions $F_1, F_2, \dots, F_k$ with domain $X$ and codomain $\left\{0,1\right\}$ and for any function $g$ with domain $\left\{0,1\right\}^k$ and codomain $\left\{0,1\right\}$, we have, \[\optstd(\compose(F_1, F_2, \dots, F_k, g)) \leq (1+o(1)) \log_2(k)\sum_{i = 1}^k \optstd(F_i).\]
\end{cor}

In the next result, we show that the upper bound in Corollary~\ref{cor:compose} is within a constant factor of optimal whenever the values of $\optstd(F_i)$ are approximately equal for all $1 \le i \le k$.

\begin{thm}
For any positive integers $k$ and $M$ such that $k$ is a power of $2$, there exist $k$ families $F_1,F_2,\dots,F_k$ with domain $X$ and codomain $\left\{0,1\right\}$ and a function $g$ with domain $\left\{0,1\right\}^k$ and codomain $\left\{0,1\right\}$ such that $\optstd(F_i)\leq M$ for all $i\in\{1,2,\dots,k\}$ and
\[\optstd(\compose(F_1, F_2, \dots, F_k, g)) \geq \frac12kM\log_2(k).\]
\end{thm}

\begin{proof}
Let $T$ be a set with $kM$ elements and $x\notin T$. Choose the domain $X=(T\cup\{x\})^k$. Let $f_{i,S}$ be the function that returns $1$ on each tuple in $X$ whose $i^{\text{th}}$ term is in $S$ and $0$ on all other tuples in $X$, where $S$ is a subset of $T$. We select $F_i=\{f_{i,S}:S\subseteq T,|S|=M\}$. It is clear that $\optstd(F_i)\leq M$ due to the learner's strategy of always guessing $0$ except when previous feedback makes an output of $0$ impossible. We select $g$ to be the bitwise OR function.

We now supply the adversary's recursive strategy to force the learner to make at least $\frac12kM\log_2(k)$ mistakes. For ease of reference, we let $T=\{1,2,\dots,kM\}$ at the top level of the recursion. If $k=1$, then the adversary's task is already done because $\log_2(k)=0$. Now assume otherwise. The adversary first restricts the space of allowed functions to the functions $g(f_{1,S_1},f_{2,S_2},\dots,f_{k,S_k})$ such that $S_1,S_2,\dots,S_k$ partition $T$. Since $x \notin T$, note that $x \notin S_i$ for each $i = 1, 2, \dots, k$. We claim that even after the restriction, the adversary can still force $\frac 12kM\log_2 k$ mistakes. We assume the equivalent statement for $\frac k2$ as an inductive hypothesis. In the first group of inputs, the adversary supplies the input $(2i-1,2i-1,\dots,2i-1,2i,2i,\dots,2i)$ for each $1\leq i\leq \frac{kM}2$, where $2i-1$ and $2i$ each appear $\frac k2$ times. After each guess, the adversary says that the learner was wrong. The adversary then partitions $T$ into sets $T_1$ and $T_2$ with $\frac{kM}2$ elements each. If the learner guessed $0$ on input $i$, then $2i-1$ goes in $T_1$ and $2i$ goes in $T_2$. If the learner guessed $1$ on input $i$, then vice versa.

Note that any function $g(f_{1,S_1},f_{2,S_2},\dots,f_{k,S_k})$ with the input $(2i-1,2i-1,\dots,2i-1,2i,2i,\dots,2i)$ has an output of $1$ whenever the statement \[\left((2i-1) \in \bigcup_{1 \le i \le k/2} S_i\right) \vee \left((2i) \in \bigcup_{k/2+1 \le i \le k} S_i\right)\] is true and it has an output of $0$ otherwise. Thus, a function $g(f_{1,S_1},f_{2,S_2},\dots,f_{k,S_k})$ is consistent with the adversary's responses provided that $S_1,\dots,S_{k/2}$ partition $T_1$ and that $S_{k/2+1},\dots,S_k$ partition $T_2$. Therefore, the adversary can now split the learner's task into two which are equivalent to the original, except with $k$ replaced by $\frac k2$. 

In particular, let $i \in \left\{1, \dots, \log_2(k) \right\}$ and suppose that before the $i^{\text{th}}$ group of inputs, the adversary obtains a partition $\left\{T_b\right\}_{b \in \left\{1,2\right\}^{i-1}}$ of $T$ for which each $T_b$ contains $\frac{k M}{2^{i-1}}$ elements. For the $i^{\text{th}}$ group of inputs, the adversary supplies $\frac{kM}2$ inputs which each have two distinct numbers, with each number occurring in $\frac{k}{2^{i}}$ consecutive coordinates and $x$ occurring in the other $k\left(1-\frac{1}{2^{i-1}}\right)$ coordinates. Specifically, for each positive integer $j \le 2^{i-1}$, the $\left(\frac{(j-1)k M}{2^i}+1\right)^{\text{st}}$ through $\left(\frac{j k M}{2^i}\right)^{\text{th}}$ inputs in the $i^{\text{th}}$ group of inputs have $x$ in all coordinates except those with indices $\frac{(j-1)k}{2^{i-1}}+1, \dots, \frac{j k}{2^{i-1}}$. The two distinct numbers in the $\left(\frac{(j-1)k M}{2^i}+r\right)^{\text{th}}$ input are the $(2r-1)^{\text{st}}$ and $(2r)^{\text{th}}$ elements of $T_{b(j)}$ in increasing order, where $b(j)$ is the element of $\left\{1,2\right\}^{i-1}$ obtained from the binary string of length $i-1$ representing $j-1$ by replacing every $0$ with $1$ and every $1$ with $2$. The learner's guesses and the adversary's responses allow the adversary to partition each $T_b$ into $T_{b 1}$ and $T_{b 2}$, each with $\frac{k M}{2^i}$ elements.

By the inductive hypothesis, the groups of inputs after the first group allow the adversary to force $2 \cdot \frac 12\left(\frac k2\right)M\log_2\left(\frac k2\right)$ more mistakes. Adding the original $\frac{kM}2$ from the first group, we get a total of
\[\frac{kM}2+2\cdot\frac 12\left(\frac k2\right)M\log_2\left(\frac k2\right)=\frac12kM\log_2(k)\]
mistakes, as desired.
\end{proof}

In the last proof, note that the element $x$ was used as a ``null" input in layers of the recursion other than the top layer. For example, if $k=8$ and $M=1$, a possible series of exchanges between the adversary and learner is as follows:

The adversary supplies the $4$ inputs $(1,1,1,1,2,2,2,2),\dots,(7,7,7,7,8,8,8,8)$ in the top layer of the recursion. The learner guesses $0$ on each, and the adversary claims that the correct answer on each is actually $1$. Thus, $T_1=\{1,3,5,7\}$ and $T_2=\{2,4,6,8\}$.

In the next layer of the recursion, the adversary supplies the $4$ inputs $(1,1,3,3,x,x,x,x)$, $(5,5,7,7,x,x,x,x)$, $(x,x,x,x,2,2,4,4)$, and $(x,x,x,x,6,6,8,8)$. Suppose again that the learner guesses $0$ on each and the adversary claims that the correct answer on each is actually $1$.

In the third and final layer, the adversary supplies the $4$ inputs $(1,5,x,x,x,x,x,x)$, $(x,x,3,7,x,x,x,x)$, $(x,x,x,x,2,6,x,x)$, and $(x,x,x,x,x,x,4,8)$. No matter what outputs the learner guesses, the adversary can deny all of them.

\section{Discussion}\label{s:discuss}

In the bandit learning scenario, we showed that \[\frac{1}{8 \ln{2}}k \ln{k} (1-o(1)) \le \optbs(k,2) \le k \ln{k} (1+o(1)).\] Note that there is a multplicative gap of $8 \ln{2}$ between the upper and lower bounds. It remains to determine the exact value of $\optbs(k, M)$ for all $M \ge 2$. Even finding the value of \[\lim_{M \rightarrow \infty} \frac{\optbs(k, M)}{M}\]  is an open problem for all fixed $k \ge 2$. By Corollary~\ref{fekete}, this limit exists for all fixed $k \ge 1$. Determining whether the limit \[\lim_{k \rightarrow \infty} \frac{\optbs(k, M)}{k \ln{k}}\] exists is also an open problem for all fixed $M \ge 2$.

For agnostic feedback, we showed that \[\optag(F,\eta) \le k \ln{k}(1+o(1))(\optstd(F)+\eta)\] when the codomain of the functions in $F$ is $\left\{0,1, \dots,k-1\right\}$. On the other hand, we found families of functions $F$ with codomain $\left\{0,1, \dots,k-1\right\}$ for which \[\optag(F,\eta)\geq(0.5-o(1))((k\ln k)\optstd(F)+k\eta).\] It remains to determine the maximum possible value of $\optag(F, \eta)$ for families of functions $F$ with codomain $\left\{0,1, \dots,k-1\right\}$ and $\optstd(F) = M$, as a function of $M$, $k$, and $\eta$. In particular, if we let $\optas(k, M, \eta)$ denote the maximum possible value of $\optag(F, \eta)$ for families of functions $F$ with codomain $\left\{0,1, \dots,k-1\right\}$ and $\optstd(F) = M$, then what is the value of \[\lim_{k \rightarrow \infty} \frac{\optas(k, M, \eta)}{k \ln{k}}\] for all fixed $M, \eta \ge 1$?

In the delayed ambiguous reinforcement scenario, we improved the upper bound on $\optambr(F)$ of $(1+o(1))(k^r r\ln{k}) \optstd(F)$ from Feng et al. \cite{feng} for families of functions $F$ with codomain $\left\{0,1,\dots,k-1\right\}$ to $(1+o(1))(k^r \ln{k})\optstd(F)$, where the $o(1)$ is with respect to $k$. This is sharp up to the leading term when $\optstd(F) \ge 2r$, using an adversary strategy from Feng et al. \cite{feng}. Furthermore, we showed that the maximum possible value of both $\optambr(F)$ and $\optbs(\cartr(F))$ is $\Theta(k^{\optstd(F)})$ over all families $F$ of functions with codomain of size $k$ and $\optstd(F) \le r$, where the constants in the upper and lower bounds depend on $r$ and $\optstd(F)$. There is a multiplicative gap of at most $r^{\optstd(F)}$ between our upper and lower bounds for this maximum.

As for closure bounds, we showed that \[\optstd(\compose(F_1, F_2, \dots, F_k, G)) \le (1+o(1))(\log_2{k})(\optstd(G)+\sum_{i = 1}^k \optstd(F_i)),\] \[\optstd(\compose(F_1, F_2, \dots, F_k, g)) \le (1+o(1))(\log_2{k})\sum_{i = 1}^k \optstd(F_i),\] and for any positive integers $k$ and $M$ with $k$ a power of $2$, there exist $k$ families $F_1,F_2,\dots,F_k$ with domain $X$ and codomain $\left\{0,1\right\}$ and a function $g$ with domain $\left\{0,1\right\}^k$ and codomain $\left\{0,1\right\}$ such that $\optstd(F_i)\leq M$ for all $i\in\{1,2,\dots,k\}$ and \[\optstd(\compose(F_1, F_2, \dots, F_k, g)) \geq \frac12kM\log_2(k).\] Let $\optcs(k, M)$ denote the maximum possible value of $\optstd(\compose(F_1, F_2, \dots, F_k, g))$ over all families $F_1,F_2,\dots,F_k$ with some domain $X$ and codomain $\left\{0,1\right\}$ and functions $g$ with domain $\left\{0,1\right\}^k$ and codomain $\left\{0,1\right\}$ such that $\optstd(F_i)\leq M$ for all $i\in\{1,2,\dots,k\}$. The proofs of our bounds imply that \[\frac14kM\log_2(k) \le \optcs(k, M) \le (1+o(1))k M\log_2{k}.\] In particular, there is a multiplicative gap of $4$ between our upper and lower bounds on $\optcs(k, M)$.

Another possible direction for future research is online learning of special families of functions. For example, Feng et al. \cite{feng} obtained some results in this direction for families of permutations and non-decreasing functions. Geneson and Zhou \cite{gz} also found some results for families of exponential functions and polynomials, though they were studying a model of online learning for smooth functions that was introduced by \cite{kl, longsmooth}. Injections, surjections, polynomials, and exponential functions are some families that could be investigated for the models of online learning studied in \cite{auer, geneson, long} and this paper. Alternatively, it would be interesting to investigate online learning scenarios where there are computational restrictions on the learning algorithms.\\

\noindent {\bf Acknowledgements.} Most of this research was performed in PRIMES 2023. We thank the organizers for this opportunity. We also thank Idan Mehalel for helpful comments on this paper.

% Manual newpage inserted to improve layout of sample file - not
% needed in general before appendices/bibliography.

\end{document}